\documentclass{article}

\usepackage{arxiv}

\usepackage[utf8]{inputenc} 
\usepackage[T1]{fontenc}    
\usepackage{hyperref}       
\usepackage{url}            
\usepackage{booktabs}       
\usepackage{amsfonts}       
\usepackage{nicefrac}       
\usepackage{microtype}      
\usepackage{lipsum}
\usepackage{graphicx}
\graphicspath{ {./images/} }
\usepackage{setspace}

\title{A Hybrid PAC Reinforcement Learning Algorithm}

\author{
 Ashkan Zehfroosh \\
  Department of Mechanical Engineering\\
  University of Delaware\\
  Newark, DE 19716 \\
  \texttt{ashkanz@udel.edu} \\
   \And
 Herbert G. Tanner \\
  Department of Mechanical Engineering\\
  University of Delaware\\
  Newark, DE 19716 \\
  \texttt{btanner@udel.edu} \\
}

\usepackage{graphicx}

\usepackage{graphicx}
\usepackage{graphics} 
\usepackage{epsfig} 
\usepackage{color,soul}
\usepackage{amsmath} 
\usepackage{amssymb,mathtools}  
\usepackage{paralist}
\usepackage{enumitem}
\usepackage{acronym}
\usepackage{booktabs}
\usepackage{verbatim}
\usepackage{cite}
\usepackage{balance}
\usepackage{comment}
\usepackage[utf8]{inputenc}
\usepackage[english]{babel}
\usepackage{MnSymbol}
\usepackage{wasysym}
\usepackage{amsthm}

\usepackage{algorithm}
\usepackage[flushleft]{threeparttable}
\usepackage{algpseudocode}
\usepackage{varwidth}
\usepackage{xfrac}
\usepackage{hhline}
\usepackage{todonotes}
\usepackage{appendix}
\usepackage{lipsum}
\usepackage{subcaption}
\newcommand\blfootnote[1]{%
  \begingroup
  \renewcommand\thefootnote{}\footnote{#1}%
  \addtocounter{footnote}{-1}%
  \endgroup
}

\begin{document}

\setlength{\abovedisplayskip}{8pt}
\setlength{\belowdisplayskip}{9pt}

\blfootnote{This work was supported by NIH under grant \# R01HD87133}

\acrodef{hri}[\textsc{hri}]{human-robot interaction}
\acrodef{mdp}[\textsc{mdp}]{Markov decision process}
\acrodef{smdp}[\textsc{smdp}]{Semi-Markov decision process}
\acrodef{ai}[\textsc{ai}]{Artificial Intelligence}
\acrodef{ml}[\textsc{ml}]{maximum likelihood}
\acrodef{pomdp}[\textsc{pomdp}]{partially observable Markov decision process}
\acrodef{momdp}[\textsc{momdp}]{mixed observability Markov decision process}
\acrodef{nlp}[\textsc{nlp}]{natural language processing}
\acrodef{pac}[\textsc{pac}]{probably approximately correct}
\acrodef{rl}[\textsc{rl}]{reinforcement learning}
\acrodef{ddq}[\textsc{ddq}]{Dyna-Delayed Q-learning}
\acrodef{tdm}[\textsc{tdm}]{Temporal Difference Models}

\maketitle
\begin{abstract}
This paper offers a new hybrid probably approximately correct (\textsc{pac}) reinforcement learning (\textsc{rl}) algorithm for Markov decision processes (\textsc{mdp}s) that intelligently maintains favorable features of its parents. 
The designed algorithm, referred to as the Dyna-Delayed Q-learning (\textsc{ddq}) algorithm, combines model-free and model-based learning approaches while outperforming both in most cases. 
The paper includes a \textsc{pac} analysis of the \textsc{ddq} algorithm and a derivation of its sample complexity. 
Numerical results are provided to support the claim regarding the new algorithm's sample efficiency compared to its parents as well as the best known model-free and model-based algorithms in application.
\end{abstract}
\medskip


\section{Introduction}
\label{intro}
While several \ac{rl} algorithms can apply to a dynamical system modelled as a \ac{mdp}, few are \ac{pac}---namely, they can guarantee how soon the algorithm will converge to the near-optimal policy. 
Existing \ac{pac} \ac{mdp} algorithms can be broadly divided into two groups: model-based algorithms like \cite{kearns2002,brafman2002, strehl2008j,strehl2012,szita2010,lattimore2014}, and model-free Delayed Q-learning algorithms \cite{strehl2006,jin2018q,dong2019q}. 
Each group has its advantages and disadvantages. 
The goal here is to capture the advantages of both groups, while preserving \ac{pac} properties. 

Model-free \ac{rl} is a powerful approach for learning complex tasks. 
For many real-world learning problems, however, the approach is taxing in terms of the size of the necessary body of data---what is more formally referred to as its \emph{sample complexity}. 
The reason is that model-free \ac{rl} ignores rich information from state transitions and only relies on the observed rewards for learning the optimal policy \cite{pong2018}. 
A popular model-free \ac{pac} \ac{rl} \ac{mdp} algorithm is known as \emph{Delayed Q-learning} \cite{strehl2006}. 
The known upper-bound on the sample complexity of Delayed Q-learning suggests that it outperforms model-based alternatives only when the state-space size of the \ac{mdp} is relatively large \cite{strehl2009}. 

Model-based \ac{rl}, on the other hand, utilizes all information from state transitions to learn a model, and then uses that model to compute an optimal policy. 
The sample complexity of model-based \ac{rl} algorithms are typically lower than that of model-free ones \cite{nagabandi2018}; the trade-off comes in the form of computational effort and possible bias \cite{pong2018}.

A popular model-based \ac{pac} \ac{rl} \ac{mdp} algorithms is \emph{R-max} \cite{brafman2002}. 
The derived upper-bound for the sample complexity of the R-max algorithm \cite{kakade2003} suggests that this model-based algorithm shines from the viewpoint of sample efficiency when the size of the state/action space is relatively small. 
This efficiency assessment can typically be generalized to most model-based algorithms.

Overall, R-max and Delayed Q-learning are incomparable in terms of their bound on the sample complexity. 
For instance, \emph{for the same sample size},
R-max is bound to return a policy of higher accuracy compared to Delayed Q-learning, while the latter will converge much faster on problems with large state spaces.

Typically, model-free algorithms circumvent the model learning stage of the solution process, a move that by itself reduces complexity in problems of large size. 
In many applications, however, model learning is not the main complexity bottleneck.  
Neurophysiologically-inspired hypotheses \cite{lee2014} have suggested that the brain approach toward complex learning tasks can be model-free (trial and error) or model-based (deliberate planning and computation) or even combination of both, depending on the amount and reliability of the available information.
This intelligent combination is postulated to contribute in making the process efficient and fast. 
The design of the \ac{pac} \ac{mdp} algorithm presented in this paper is motivated by such observations.
Rather than following strictly one of the two prevailing directions, it orchestrates a marriage of a model-free (Delayed Q-learning) with a model-based (R-max) \ac{pac} algorithm, in order to give rise to a new \ac{pac} algorithm (\ac{ddq}) that combines the advantages of both.

The search for a connection between model-free and model-based \ac{rl} algorithms begins with the Dyna-Q algorithm \cite{sutton1991}, in which synthetic generated experiences based on the learned model are used to expedite Q-learning. 
Some other examples that continued along this thread of research are partial model back propagation \cite{heess2015}, training a goal condition Q function \cite{parr2008,sutton2011,schaul2015,andrychowicz2017}, and integrating model-based linear quadratic regulator based algorithm into model-free framework of path integral policy improvement \cite{chebotar2017}. The recently introduced \ac{tdm} provides a smooth(er) transition from model-free to model-based,  during the learning process \cite{pong2018}. What is missed in the literature is a \ac{pac} combination of model-free and model-based frameworks. 

Here the Dyna-Q idea is leveraged to combine two popular \ac{pac} algorithms, one model-free and one model-based, into a new one named \ac{ddq}, which is not only \ac{pac} like its parents, but also inherits the best of both worlds: it will intelligently behave more like a model-free algorithm on large problems, and operate more like a model-based algorithm on problems that require high accuracy, being content with the smallest among the sample sizes required by its parents.
Specifically, the sample complexity of  \ac{ddq}, in the worst case, matches the minimum bound between that of R-max and Delayed Q-learning, and often outperforms both. 
Note that \ac{ddq} algorithm is purely online and does not assume accessing to a generative model like in \cite{azar2013}. 
While the provable worst case upper bound on the sample complexity of \ac{ddq} algorithm is higher than the best known model-based~\cite{szita2010} and model-free~\cite{jin2018q,dong2019q} algorithms, we can demonstrate (see Section \ref{sim}) that the hybrid nature allows for superior performance of the \ac{ddq} algorithm in applications. 
The availability of a hybrid \ac{pac} algorithm like \ac{ddq} in hand, obviates the choice between a model-free and a model-based approach. 

Our own motivation for developing of this new breed of \ac{rl} algorithms comes from application problems in the area of early pediatric motor rehabilitation, where robots can be used as smart toys to socially interact with infants who have special needs, and engage with them socially in play-based activity that involves gross motion. 
There, \ac{mdp} models can be constructed to capture the dynamics of the social interaction between infant and robot, and \ac{rl} algorithms can guide the behavior of the robot as it interacts with the infant in order to achieve the maximum possible rehabilitation outcome---the latter possibly quantified by the overall length of infant displacement, or the frequency of infant motor transitions. 
Some early attempts at modeling such instances of \ac{hri} did not result in models of particularly large state and action spaces, but were particularly complicated by the absence of sufficient data sets for learning~\cite{zehfroosh2017,zehfroosh2018}. 
This is because every child is different, and the exposure of each particular infant to the smart robotic toys (during which \ac{hri} data can be collected) is usually limited to a few hours per month.
There is a need, therefore, for reinforcement learning approaches that can maintain efficiency and accuracy even when the learning set is particularly small.

The paper starts with some technical preliminaries in Section \ref{pre}. 
This section introduces the required properties of a \ac{pac} \ac{rl} algorithm in the form of a well-known theorem. 
The \ac{ddq} algorithm is introduced in Section \ref{ddq}, with particular emphasis given on its update mechanism. 
Section~\ref{s} presents the mathematical analysis that leads the establishment of the algorithm's \ac{pac} properties, and the analytic derivation of its sample complexity. 
Finally, Section \ref{sim} offers numerical data to support the theoretical sample complexity claims, through an illustrative grid-world example. 
The data indicate that \ac{ddq} outperforms Delayed Q-learning and R-max in terms of the required samples to learn near-optimal policy. 
To promote readability, the proofs of most of the  lemmas supporting the proof of our main result are included separately in the paper's Appendix.
\medskip


\section{Technical Preliminaries}
\label{pre}
A finite \ac{mdp} $M$ is a tuple $\{S,A,R,T,\gamma\}$ with elements 

\begin{center}
\begin{threeparttable}
\begin{tabbing}
\hspace*{3.5cm} \= \kill 
$S$ \> a finite set of states\\
$A$ \> a finite set of actions\\
$R:S\times A \to [0,1]$ \> the \emph{reward} from executing $a$ at  $s$\\
$T:S\times A\times S \to [0,1]$ \> the \emph{transition probabilities} \\
$\gamma \in [0,1)$ \> the \emph{discount factor}
\end{tabbing}
\end{threeparttable}
\end{center}
\medskip

A \emph{policy} $\pi$ is a mapping $\pi : S \to A$ that selects an action $a$ to be executed at state $s$.
A policy is \emph{optimal} if it maximizes the expected sum of discounted rewards; if $t$ indexes the current time step and $a_t$, $s_t$ denote current action and state, respectively, then this expected sum is written $\mathbb{E}_{M}\big\{\sum_{t=0}^{\infty}\gamma^{t} R(s_{t},a_{t}) \big\}$. 
The discount factor $\gamma$ here reflects the preference of immediate rewards over future ones. 
The \emph{value} of state $s$ under policy $\pi$ in \ac{mdp} $M$ is defined as
\begin{equation*} \label{value}
v_{M}^{\pi}(s)=\mathbb{E}_{M}\left\{ R\big(s,\pi(s)\big) + \sum_{t=1}^{\infty}\gamma^{t} R\big(s_{t},\pi(s_t)\big) \right\} 
\end{equation*}
Note that an upper bound for the value at any state is $v_{\max} = \tfrac{1}{1-\gamma}$.
Similarly defined is the value of \emph{state-action pair} $(s,a)$ under policy $\pi$:
\begin{equation*}  \label{qvalue}
Q_{M}^{\pi}(s,a)=\mathbb{E}_{M}\left\{ R(s,a) + \sum_{t=1}^{\infty}\gamma^{t} R\big(s_{t},\pi(s_t)\big) \right\} 
\end{equation*}
Every \ac{mdp} $M$ has at least one optimal policy $\pi^*$ that results in an optimal value (or state-action value) assignment at all states; the latter is denoted $v_{M}^*(s)$ (or $Q_{M}^*(s,a)$, respectively). 

The standard approach to finding the optimal values is through the search for a fix point of the Bellman equation
\begin{equation*} \label{bellman-value}
v_{M}^{*}(s)= \max_a \{ R(s,a) + \gamma \sum_{s'}T(s,s',a) v_{M}^{*}(s') \}
\end{equation*}
which, after substituting $V_{M}^{*}(s')=\max_a Q_{M}^{*}(s',a)$, can equivalently be written in terms of state-action values
\begin{equation*}  \label{bellman-Q}
Q_{M}^{*}(s,a)=  R(s,a) + \gamma \sum_{s'} T(s,s',a) v_{M}^{*}(s') 
\end{equation*}

Reinforcement learning, (\textsc{rl}) is a procedure to obtain an optimal policy in an \ac{mdp}, when the actual transition probabilities and/or reward function are not known. 
The procedure involves exploration of the \ac{mdp} model. 
An \ac{rl} algorithm usually maintains a table of state-action pair value estimates $Q(s,a)$ that are updated based on the exploration data. 
We denote $Q_{t}(s,a)$ the currently stored value for state-action pair $(s,a)$ at timestep $t$ during the execution of an \ac{rl} algorithm. 
Consequently, $v_{t}(s)=\max_a Q_t(s,a)$. 
An \ac{rl} algorithm is \emph{greedy} if it at any timestep $t$, it always executes action $a_t = \mathrm{arg} \max_{a \in A} Q_t(s_t,a)$.
The policy in force at time step $t$ is similarly denoted $\pi_t$.
In what follows, we denote $|S|$ the cardinality of a set $S$.

Reinforcement learning algorithms have been classified as \emph{model-based} or \emph{model-free}. 
Although the characterization is debatable, what is meant by calling an \ac{rl} algorithm ``model-based,'' is that $T$ and/or $R$ are estimated based on online observations (exploration data), and the resulting estimated model subsequently informs the computation of the the optimal policy.
A model-free \ac{rl} algorithm, on the other hand, would skip the construction of an estimated \ac{mdp} model, and search directly for an optimal policy over the policy space.
An \ac{rl} algorithm is expected to converge to the optimal policy, practically reporting a near-optimal one at termination.

Probably approximately correct (\ac{pac}) analysis of \ac{rl} algorithms deals with the question of how fast an \ac{rl} algorithm converges to a near-optimal policy. 
An \ac{rl} algorithm is \ac{pac} if there exists a probabilistic bound on the number of exploration steps that the algorithm can take before converging to a near-optimal policy. 

\newtheorem{definition}{Definition}

\medskip
\begin{definition}
Consider that an \ac{rl} algorithm $\mathcal{A}$ is executing on \ac{mdp} $M$. Let $s_t$ be the visited state at time step $t$ and $\mathcal{A}_t$ denotes the (non-stationary) policy that the $\mathcal{A}$ executes at $t$. 
For a given $\epsilon > 0$ and $\delta >0$, $\mathcal{A}$ is a \emph{\ac{pac} \ac{rl} algorithm} if there is an $N>0$ such that with probability at least $1-\delta$ and for all but $N$ time steps,
\begin{equation}
v^{\mathcal{A}_t}_M(s_t) \geq v^{*}_M(s_t)-\epsilon \qquad
\label{epsilon-optimality}
\end{equation}
\end{definition}
\medskip

Equation \eqref{epsilon-optimality} is known as the $\epsilon$-optimality condition and $N$ as the \emph{sample complexity} of $\mathcal{A}$, which  is a function of $\big(|S|,|A|,\frac{1}{\epsilon},\frac{1}{\delta},\frac{1}{1-\gamma}\big)$.

\medskip
\begin{definition}
Consider \ac{mdp} $M= \{S,A,R,T,\gamma\}$ which at time $t$ has a set of state-action value estimates $Q_t(s,a)$, and let $K_t \subseteq S \times A$ be a set of state-action pairs labeled \emph{known}. 
The \emph{known state-action \ac{mdp}} 
\[
M_{K_t}= \big\{S \cup \{z_{s,a} | (s,a) \notin K_t \},A,T_{K_t},R_{K_t},\gamma\big\} 
\]
is an \ac{mdp} derived from $M$ and $K_t$ by defining new states $z_{s,a}$ for each  unknown state-action pair $(s,a) \notin K_t$, with self-loops for all actions, i.e., $T_{K_t}(z_{s,a},\cdot,z_{s,a})= 1$.
For all $(s,a)\in K_t$, it is  $R_{K_t}(s,a) = R(s,a)$ and  $T_{K_t}(s,a,\cdot) = T(s,a,\cdot)$.
When an unknown state-action pair $(s,a) \notin K_t$ is experienced, $R_{K_t}(s,a) = Q_t(s,a)(1-\gamma)$ and the model jumps to $z_{s,a}$ with $T_{K_t}(s,a,z_{s,a})=1$;
subsequently, $R_{K_t}(z_{s,a},\cdot)=Q_t(s,a)(1-\gamma)$.  
\end{definition}
\medskip

Let $K_t$ be set of current known state-action pairs of an \ac{rl} algorithm $\mathcal{A}$ at time $t$, and allow
$K_t$ to be arbitrarily defined as long as it depends only on the history of exploration data up to $t$. 
Any $(s,a) \notin K_t$ experienced at time step $t$ marks an \emph{escape event}. 
\medskip

\newtheorem{theorem}{Theorem}
\begin{theorem}[\!\!\cite{strehl2009}]
\label{pac}
Let $\mathcal{A}$ be a greedy \ac{rl} algorithm for an arbitrary \ac{mdp} $M$, and let $K_t$ be the set of current known state-action pairs, defined based only on the history of the exploration data up to timestep $t$. 
Assume that $K_t = K_{t+1}$ unless an update to some state-action value occurs or an escape event occurs at timestep $t$, and that $Q_t(s,a) \leq v_\mathrm{max}$ for all $(s,a)$ and $t$. 
Let $M_{K_t}$ be the known state-action \ac{mdp} at timestep $t$ and $\pi_{t}(s)=\mathrm{arg}\max_a Q_t(s,a)$ denote the greedy policy that $\mathcal{A}$ executes. 
Suppose now that for any positive constant $\epsilon$ and $\delta$, the following conditions hold with probability at least $1-\delta$ for all $s$, $a$ and  $t$:
\begin{description}
    \item[optimism:] \qquad $v_t(s) \geq v^{*}_M(s)-\epsilon$ 
    \item[accuracy:] \qquad $v_t(s)-v^{\pi_t}_{M_{K_t}}(s) \leq \epsilon$ 
    \item[complexity:] \hspace{5ex} sum of number of timesteps with $Q$-value updates plus number of timesteps with escape events is bounded by $\zeta(\epsilon,\delta)>0$. 
\end{description}
Then, executing algorithm $\mathcal{A}$ on any \ac{mdp} $M$ will result in following a $4\epsilon$-optimal policy on all but
\begin{equation}
\label{learningcom}
    \mathcal{O} \left( \frac{\zeta(\epsilon,\delta)}{\epsilon (1-\gamma)^2} \ln{(\tfrac{1}{\delta})} \ln{(\tfrac{1}{\epsilon (1-\gamma)})} \right) \simeq \mathcal{O} \left( \frac{\zeta(\epsilon,\delta)}{\epsilon (1-\gamma)^2} \right)
\end{equation}
timesteps, with probability at least $1-2\delta$.
\end{theorem}
\medskip


\section{DDQ Algorithm} 
\label{ddq}
This section presents Algorithm~\ref{alg:ddq}, the one we call \ac{ddq} and the main contribution of this paper. 
\textsc{Ddq} integrates elements of R-max and Delayed $Q$-learning, while preserving the implementation advantages of both.
We refer to the assignment in line $31$ of Algorithm~\ref{alg:ddq} as a \emph{type-$1$ update}, and to the one on line $52$ as a \emph{type-$2$ update}. 
Type-$1$ updates use the $m_1$ most recent experiences (occurances) of a state-action pair $(s,a)$ to update that pair's value, while a type-$2$ update is realized through a value iteration algorithm (lines $43-54$) and applies to state-action pairs experienced at least $m_2$ times. 
The outcome at timestep $t$ of the value iteration for a type-2 update is denoted $Q^\mathrm{vl}_t(s,a)$.
The value iteration is set to run for $\frac{\ln{(1/(\epsilon_2(1-\gamma)))}}{(1-\gamma)}$ iterations; parameter $\epsilon_2$ regulates the desired accuracy on the resulting estimate (Lemma~\ref{vi}).
A type-$1$ update is successful only if the condition on line $30$ of the algorithm is satisfied, and this condition ensures that the type-1 update necessarily decreases the value estimate by at least $\epsilon_1=3\epsilon_2$.
Similarly, a type-$2$ update is successful only if the condition on line $51$ of the algorithm holds. 
The \ac{ddq} algorithm maintains the following internal variables: \begin{itemize}[leftmargin=0.1in]
    \item $l(s,a)$:  the number of samples gathered for the update type-$1$ of $Q(s,a)$ once $l(s,a)=m_1$.
    \item $U(s,a)$:  the running sum of target values used for a type-$1$ update of $Q(s,a)$, once enough samples have been gathered.
    \item $b(s,a)$:  the timestep at which the most recent or ongoing collection of $m_1$ $(s,a)$ experiences has  started.
    \item $\mathsf{learn}(s,a)$: a Boolean flag that indicates whether or not samples are being gathered for type-$1$ update of $Q(s,a)$. The flag is set to $\mathrm{true}$ initially, and is reset to $\mathrm{true}$ whenever some Q-value is updated. 
    It flips to $\mathrm{false}$ when no updates to any Q-values occurs within a time window of $m_1$ experiences of $(s,a)$ in which  attempted updates type-$1$ of $Q^i(s,a)$ fail.
    \item $n(s,a)$: variable that keeps track of the number of times $(s,a)$ is experienced. \item $n(s,a,s')$: variable that keeps track of the number of transitions to $s'$ on action $a$ at state $s$.
    \item $r(s,a)$: the accumulated rewards by doing $a$ in $s$. 
\end{itemize}

The execution of the \ac{ddq} algorithm is tuned via the $m_1$ and $m_2$ parameters. 
One can practically reduce it to Delayed $Q$-learning by setting $m_2$ very large, and to R-max by setting $m_1$ large.
The next section provides a formal proof that \ac{ddq} is not only \ac{pac} but also \emph{possesses the minimum sample complexity} between R-max and Delayed $Q$-learning in the worst case ---often, it outperforms both.

\begin{algorithm}
\caption{The \ac{ddq} algorithm \label{alg:ddq}}
\begin{algorithmic}[1]
{\normalsize{
\State \textbf{Inputs}: $S,A,\gamma,m_1,m_2,\epsilon_1,\epsilon_2$
\For{\textbf{all} $s,a,s'$}
\State $Q(s,a) \gets v_\mathrm{max}$ \Comment{{\footnotesize{initialize $Q$ values to its maximum}}}
\State $U(s,a) \gets 0$ \Comment{{\footnotesize{used for attempted updates of type-$1$}}}
\State $l(s,a) \gets 0$ \Comment{{\footnotesize{counters}}}
\State $b(s,a) \gets 0$ \Comment{{\footnotesize{beginning timestep of attempted update type-$1$}}}
\State $\mathsf{learn}(s,a) \gets \mathrm{true}$ \Comment{{\footnotesize{learn flags}}}
\State $n(s,a) \gets 0$ \Comment{{\footnotesize{number of times $(s,a)$ is tried}}}
\State $n(s,a,s') \gets 0$ \Comment{{\footnotesize{number of transitions to $s'$ by $a$ in $s$}}}
\State $r(s,a) \gets 0$ \Comment{{\footnotesize{accumulated reward by execution of $a$ in $s$}}}
\EndFor
\State $t^{*} \gets 0$ \Comment{{\footnotesize{time of the most recent successful timestep}}}
\For{$t=1,2,3,...$}
\State let $s$ denotes the state at time $t$
\State choose action $a=\mathrm{arg} \max_{a' \in A}Q(s,a')$
\State observe immediate reward $r$ and next state $s'$
\State $n(s,a)=n(s,a)+1$  
\State $n(s,a,s')=n(s,a,s')+1$ 
\State $r(s,a)=r(s,a)+r$
\If{$b(s,a) \leq t^{*}$}
\State $\mathsf{learn}(s,a) \gets \mathrm{true}$
\EndIf
\If{$\mathsf{learn}(s,a)=\mathrm{true}$}
\If{$l(s,a)=0$}
\State $b(s,a) \gets t$
\EndIf
\State $l(s,a) \gets l(s,a)+1$
\State $U(s,a) \gets U(s,a)+r+\gamma \max_{a'}Q(s',a')$
\If{$l(s,a)=m_1$}
\If{$Q(s,a)-U(s,a)/m_1 \geq 2 \epsilon_1$}
\State $Q(s,a) \gets U(s,a)/m_1 + \epsilon_1$ \ and \ $t^{*} \gets t$ 
\ElsIf{$b(s,a) > t^{*}$}
\State $\mathsf{learn}(s,a) \gets \mathrm{false}$
\EndIf
\State $U(s,a) \gets 0$ \ and \ $l(s,a) \gets 0$
\EndIf
\EndIf
\If{$n(s,a)=m_2$}
\State $t^{*} \gets t$
\For{\textbf{all} $(\overline{s},\overline{a})$}
\State $Q_\mathrm{vl}(\overline{s},\overline{a}) \gets Q(\overline{s},\overline{a})$
\EndFor
\For{$i=1,2,3,...,(\frac{\ln{(1/(\epsilon_2(1-\gamma)))}}{(1-\gamma)})$}
\For{\textbf{all} $(\overline{s},\overline{a})$}
\If{$n(\overline{s},\overline{a}) \geq m_2$}
\State \begin{varwidth}[t]{\linewidth}  
$Q_\mathrm{vl}(\overline{s},\overline{a}) \gets \frac{r(\overline{s},\overline{a})}{n(\overline{s},\overline{a})} +  \gamma \sum_{s''} \frac{n(\overline{s},\overline{a},s'')}{n(\overline{s},\overline{a})} \max_{a'}Q_\mathrm{vl}(s'',a')$
\end{varwidth}
\EndIf
\EndFor
\EndFor
\For{\textbf{all} $(\overline{s},\overline{a})$}
\If{$Q_\mathrm{vl}(\overline{s},\overline{a}) \leq Q(\overline{s},\overline{a})$}
\State $Q(\overline{s},\overline{a}) \gets Q_\mathrm{vl}(\overline{s},\overline{a})$
\EndIf
\EndFor
\EndIf
\EndFor
}}
\end{algorithmic}
\end{algorithm}
\medskip


\section{PAC Analysis of DDQ Algorithm} 
\label{s}
In general, the sample complexity of R-max and Delayed $Q$-learning is incomparable~\cite{strehl2009}; the former is better in terms of the accuracy of the resulting policy while the latter is better in terms of scaling with the size of the state space.   
The sample complexity of R-max algorithm is $\frac{|S|^2|A|}{\epsilon^3(1-\gamma)^8}$ ---note the power on $\epsilon$; the sample complexity of Delayed $Q$-learning algorithm is $\frac{|S||A|}{\epsilon^4(1-\gamma)^8}$ ---note the linear scaling with $|S|$. 
It appears that \ac{ddq} can bring together the best of both worlds; its sample complexity is
\[
 \mathcal{O} \left( \min\left\{\mathcal{O}\big(\tfrac{|S|^2|A|}{\epsilon^3(1-\gamma)^8} \big), \mathcal{O}\big(\tfrac{|S||A|}{\epsilon^4(1-\gamma)^8} \big) \right\} \right)
\]

Before formally stating the \ac{pac} properties of the \ac{ddq} algorithm and proving the bound on its sample complexity, some technical groundwork needs to be laid. 
To slightly simplify notation, let $\kappa \triangleq |S||A|(1+\frac{1}{(1-\gamma)\epsilon_1})$. Moreover, subscript $t$ marks the value of a variable at the beginning of timestep $t$ (particularly line $23$ of the algorithm). 

\medskip
\begin{definition}
An event when $\mathsf{learn}(s,a)=\mathrm{true}$ and  at the same time $l(s,a)=m_1$ or $n(s,a)=m_2$, is called an \emph{attempted update}.
\end{definition}

\medskip
\begin{definition}
\label{definition:beta}
At any timestep $t$ in the execution of \ac{ddq} algorithm the set of \emph{known state-action pairs} is defined as:
\begin{equation*}
    K_t=\big\{ (s,a) \mid n(s,a) \geq m_2 \quad \mathrm{or} \quad
    Q_t(s,a)-\big(R(s,a)+\gamma \textstyle{\sum_{s'}T(s,a,s')v_t(s')\big)} \leq 3 \epsilon_1 \big\}
\end{equation*}
\end{definition}
\medskip

In subsequent analysis, and to distinguish between the conditions that make a state-action pair $(s,a)$ known, the set $K_t$ will be partitioned into two subsets:
\begin{align*}
    K^{1}_t =& \big\{ (s,a) \mid  Q_t(s,a)-\big(R(s,a)+\gamma \textstyle{\sum_{s'}T(s,a,s')v_t(s')\big)} \leq 3 \epsilon_1 \big\} \\
    K^{2}_t =& \big\{ (s,a) \mid n(s,a) \geq m_2 \big\}
\end{align*}

\medskip
\begin{definition}
\label{successful-t}
In the execution of \ac{ddq} algorithm a timestep $t$ is called a \emph{successful timestep} if at that step any state-action value is updated or the number of times that a state-action pair is visited reaches $m_2$. Moreover, considering a particular state-action pair $(s,a)$, timestep $t$ is called a \emph{successful timestep for $(s,a)$} if at $t$ either update type-1 happens to $Q(s,a)$ or the number of times that $(s,a)$ is visited reaches $m_2$.
\end{definition}

Recall that a type-$1$ update necessarily decreases the Q-value by at least $\epsilon_1$. 
Defining rewards as positive quantities prevents the Q-values from becoming negative. At the same time, state-action pairs can initiate update type-$2$ only once they are experienced $m_2$ times. Together, these conditions facilitate the establishment of an upper-bound on the total number of successful timesteps during the execution of \ac{ddq}:

\newtheorem{lemma}{Lemma}
\begin{lemma}
\label{sbound}
The number of successful timesteps for a particular state-action pair $(s,a)$ in a \ac{ddq} algorithm is at most $1+\frac{1}{(1-\gamma)\epsilon_1}$. Moreover, the total number of successful timesteps is bounded by $\kappa$.
\end{lemma}
\begin{proof}
See Appendix~\ref{L1}.
\end{proof}

\begin{lemma}
\label{ubound} 
The total number of attempted updates in \ac{ddq} algorithm is bounded by $|S||A|(1+\kappa)$.
\end{lemma}
\begin{proof}
See Appendix~\ref{L2}.
\end{proof}

\begin{lemma}
\label{a1}
Let $M$ be an \ac{mdp} with a set of known state-action pairs $K_t$. 
If we assume that for all state-action pairs $(s,a) \notin K_t$ we have $Q_t(s,a) \leq \frac{1}{1-\gamma}$, then for all state-action pairs in the known state-action \ac{mdp} $M_{K_t}$ it holds
\begin{equation*}
Q^{*}_{M_{K_t}}(s,a) \leq \frac{1}{1-\gamma}
\end{equation*}
\end{lemma}
\begin{proof}
See Appendix~\ref{L3}.
\end{proof}

Choosing $m_1$ big enough and applying Hoefding's inequality allows following conclusion (Lemma~\ref{type1opt}) for all type-$1$ updates, and paves the way for establishing the optimism condition of  Theorem~\ref{pac}.

\medskip
\begin{lemma}
\label{type1opt}
Suppose that at time $t$ during the execution of \ac{ddq} a state-action pair $(s,a)$ experiences a successful update of type-1 with its value changing from $Q(s,a)$ to $Q'(s,a)$, and that there exists $\exists \epsilon_2 \in (0,\tfrac{\epsilon_1}{2})$ such that $\forall s \in S$ and $\forall t'<t$, $v_{t'}(s) \geq v^{*}_M(s)-2\epsilon_2$.
If 
\begin{equation}
\label{m1}
    m_1 \geq \frac{\ln{\big(\tfrac{8|S||A|(1+\kappa)}{\delta}\big)}}{2 (\epsilon_1-2\epsilon_2)^2 (1-\gamma)^2} \simeq \mathcal{O} \left( \frac{\ln{\big(\tfrac{|S|^2|A|^2}{\delta}\big)}}{\epsilon^2_1 (1-\gamma)^2} \right)
\end{equation}
for $\kappa = |S||A|(1+\sfrac{1}{(1-\gamma)\epsilon_1})$, then $Q'(s,a) \geq Q^{*}_M(s,a)$ with probability at least $1-\frac{\delta}{8}$.
\end{lemma}
\begin{proof}
In Appendix~\ref{L4}.
\end{proof}

The following two lemmas are borrowed from \cite{strehl2009} with very minor modifications, and inform on how to choose parameter $m_2$, and the number of iterations for the value iteration part of the \ac{ddq} algorithm in order to obtain a desired accuracy.

\medskip
\begin{lemma} (cf. \cite[Proposition~4]{strehl2009})
\label{vi}
Suppose the value-iteration algorithm runs on \ac{mdp} $M$ for $\frac{\ln{(\sfrac{1}{\epsilon_2(1-\gamma)})}}{1-\gamma}$ iterations, and each state-action value estimate $Q(s,a)$ is initialized to some value between $0$ and $v_\mathrm{max}$ for all states and actions. Let $Q'(s,a)$ be the state-action value estimate the algorithm yields. Then
$
    \max_{s,a} \big\{ |Q'(s,a)-Q^{*}_M(s,a)| \big\} \leq \epsilon_2
$.
\end{lemma}


\begin{lemma}
\label{estimation}
Consider an \ac{mdp} $M$ with reward function $R$ and transition probabilities $T$. 
Suppose another \ac{mdp} $\hat{M}$ has the same state and action set as $M$, but maintains an \ac{ml} estimate of $R$ and $T$, with $n(s,a) \geq m_2$, in the form of $\hat{R}$ and $\hat{T}$ respectively.
With $C$ a constant and for all state-action pairs, choosing
\begin{equation*}
    m_2 \geq C\left(\frac{|S|+\ln{(\sfrac{8|S||A|}{\delta})}}{\epsilon^2_2(1-\gamma)^4}\right) \simeq O\left(\frac{|S|+\ln{(\sfrac{|S||A|}{\delta})}}{\epsilon^2_2(1-\gamma)^4}\right)
\end{equation*}
guarantees
\begin{align*}
    |R(s,a)-\hat{R}(s,a)| &\leq C \epsilon_2 (1-\gamma)^2\\
    \|T(s,a,\cdot)-\hat{T}(s,a,\cdot)\|_1 & \leq C \epsilon_2 (1-\gamma)^2
\end{align*}
with probability at least $1-\frac{\delta}{8}$. 
Moreover, for any policy $\pi$ and for all state-action pairs, 
\begin{align*}
    |Q^{\pi}_M(s,a)-Q^{\pi}_{\hat{M}}(s,a)| & \leq \epsilon_2 \\
    |v^{\pi}_M(s)-v^{\pi}_{\hat{M}}(s)| & \leq \epsilon_2
\end{align*}
with probability at least $1-\frac{\delta}{8}$.
\end{lemma}
\begin{proof}
Combine \cite[Lemmas~12--15]{strehl2009}.
\end{proof}

\begin{lemma}
\label{a2}
Let $t_1 < t_2$ be two timesteps during the execution of the \ac{ddq} algorithm. 
If 
\[
Q_{t_1}(s,a) \geq Q^{*}_{M_{K^{2}_{t_1}}}(s,a) - 2\epsilon_2 \enspace \forall (s,a) \in S \times A
\]
then with probability at least $1-\frac{\delta}{8}$
\[
  Q^{*}_{M_{K^{2}_{t_1}}}(s,a) \geq Q^{*}_{M_{K^{2}_{t_2}}}(s,a) \enspace \forall (s,a) \in S \times A  
\]
\end{lemma}
\begin{proof}
See Appendix~\ref{L6}.
\end{proof}

Lemmas~\ref{vi} and \ref{estimation} together have as a consequence the following Lemma, which contributes to establishing the accuracy condition of Theorem~\ref{pac} for the \ac{ddq} algorithm.

\begin{lemma}
\label{lubk2}
During the execution of \ac{ddq}, for all $t$ and $(s,a) \in S \times A$, we have:
\begin{equation}
    \label{lm8}
    Q^{*}_{M_{K^{2}_t}}(s,a) - 2\epsilon_2 \leq Q_t(s,a) \leq Q^{*}_{M_{K^{2}_t}}(s,a) + 2\epsilon_2
\end{equation}
with probability at least $1-\frac{3\delta}{8}$.
\end{lemma}
\begin{proof}
See Appendix~\ref{L8}.
\end{proof}

Lemma~\ref{sbound} has already offered a bound on the number of updates in \ac{ddq}; however, for the complexity condition of Theorem~\ref{pac} to be satisfied, one needs to show that during the execution of Algorithm \ref{alg:ddq} the number of escape events is also bounded. 
The following Lemma is the first step: it states that by picking $m_1$ as in \eqref{m1}, and under specific conditions, an escape event necessarily results in a successful type-$1$ update. 
With the number of updates bounded,  Lemma~\ref{lemma6} can be utilized to derive a bound on the number of escape events. 

\medskip
\begin{lemma}
\label{lemma6}
With the choice of $m_1$ as in \eqref{m1}, and assuming the \ac{ddq} algorithm at timestep $t$ with $(s,a) \notin K_t$, $l(s,a)=0$ and $\mathsf{learn}(s,a)=\mathrm{true}$, we know that an attempted type-1 update of $Q(s,a)$ will necessarily occur within $m_1$ occurrences of $(s,a)$ after $t$, say at timestep $t_{m_1}$. If $(s,a)$ has been visited fewer than $m_2$ till $t_{m_1}$, then the attempted type-1 update at $t_{m_1}$ will be successful with probability at least $1-\frac{\delta}{8}$. 
\end{lemma}

\begin{proof}
See Appendix~\ref{L9}.
\end{proof}

\begin{lemma}
\label{lemma7}  
Let $t$ be the timestep when  $(s,a)$ has been visited for $m_1$ times after the conditions of Lemma~\ref{lemma6} were satisfied. 
If the update at timestep $t$ is unsuccessful and at timestep $t+1$ it is $\mathsf{learn}(s,a)=\mathrm{false}$, then $(s,a) \in K_{t+1}$. 
\end{lemma}
\begin{proof}
See Appendix~\ref{L10}.
\end{proof}

A bound on the number the escape events of \ac{ddq} algorithm can be derived in a straightforward way. Note that a state-action pair that is visited $m_2$ times becomes a permanent member of set $K_t$.  
Therefore, the number of escape events is bounded by $|S||A|m_2$. 
On the other hand, Lemma~\ref{lemma6} and the $\mathsf{learn}$ flag mechanism (i.e. Lemma~\ref{lemma7}) suggest another upper bound on escape events. 
The following Lemma simply states an upper bound for escape events in \ac{ddq} as the minimum among the two bounds.

\begin{lemma}
\label{ebound}
During the execution of \ac{ddq}, with the assumption that Lemma~\ref{lemma6} holds, the total number of timesteps with $(s_t,a_t) \notin K_t$ (i.e. escape events) is at most $\min\{2 m_1 \kappa, |S||A| m_2)\}$.
\end{lemma}
\begin{proof}
See Appendix~\ref{L11}.
\end{proof}

Next comes the main result of this paper. 
The statement that follows establishes the  \ac{pac} properties of the \ac{ddq} algorithm and provides a bound on its sample complexity.

\begin{theorem}
\label{ddq pac}
Consider an \ac{mdp} $M= \{S,A,T,R,\gamma\}$, and let $\epsilon \in (0,\tfrac{1}{1-\gamma})$, and $\delta \in (0,1)$. 
There exist $m_1= \mathcal{O}\left(\sfrac{\ln{(\sfrac{|S|^2|A|^2}{\delta})}}{\epsilon^2_1(1-\gamma)^2}\right)$ and $m_2= \mathcal{O}\left(\sfrac{|S|+\ln{(\sfrac{|S||A|}{\delta})}}{\epsilon^2_2(1-\gamma)^4}\right)$ with $\frac{1}{\epsilon_1}=\tfrac{3}{(1-\gamma)\epsilon}=\mathcal{O}\left(\sfrac{1}{\epsilon(1-\gamma)}\right)$ and $\epsilon_2=\frac{\epsilon_1}{3}$, such that if \ac{ddq} algorithm is executed, $M$ follows a $4 \epsilon$-optimal policy with probability at least $1-2 \delta$ on all but 
\begin{equation*}
    \mathcal{O} \left( \min\big\{\mathcal{O}(\sfrac{|S|^2|A|}{\epsilon^3(1-\gamma)^8} ), \mathcal{O}(\sfrac{|S||A|}{\epsilon^4(1-\gamma)^8} ) \big\} \right)
\end{equation*}
timesteps (logarithmic factors ignored).
\end{theorem}

\begin{proof}
We intend to apply Theorem~\ref{pac}. 
To satisfy the \emph{optimism} condition, we start by proving that $Q_t(s,a) \geq Q^{*}_M(s,a) - 2\epsilon_2$ by strong induction for all state-action pairs: 
\begin{inparaenum}[(i)]
\item At $t=1$, the value of all state-action pairs are set to the maximum possible value in \ac{mdp} $M$. 
This implies that $Q_1(s,a) \geq Q^{*}_M(s,a) \geq Q^{*}_M(s,a) - 2\epsilon_2$, therefore  $v_t(s) \geq v^{*}_M(s) - 2\epsilon_2$. 
\item Assume that $Q_t(s,a) \geq Q^{*}_M(s,a) - 2\epsilon_2$ holds for all timesteps before or equal to $t=n-1$. 
\item At timestep $t=n$, all $(s,a) \notin K^{2}_n$ can only be updated by a type-$1$ update before or at $t=n$. 
For these state-action pairs, Lemma~\ref{type1opt} implies that it will be $Q_n(s,a) \geq Q^{*}_M(s,a)$ with probability $1-\frac{\delta}{8}$.

For all $(s,a) \in K^{2}_n$, on the other hand, by Lemma~\ref{lubk2} and with probability $1-\frac{3\delta}{8}$:
\end{inparaenum}
\begin{equation*}
    Q_n(s,a) \geq Q^{*}_{M_{K^{2}_n}}(s,a) - 2\epsilon_2 \geq Q^{*}_M(s,a) -2\epsilon_2
\end{equation*}
Note that $Q^{*}_{M_{K^{2}_n}}(s,a) \geq Q^{*}_M(s,a)$ since $M_{K^2_n}$ is similar to $M$ exept for $(s,a) \notin K^2_n$ which their values are set to be more than or equal to $Q^*_M(s,a)$. Therefore, $Q_t(s,a) \geq Q^{*}_M(s,a) - 2\epsilon_2$ holds for all timesteps $t$ and all state-action pairs, which directly implies $v_t(s) \geq v^{*}_M(s) - 2\epsilon_2 \geq v^{*}_M(s) - \epsilon$.\\

To establish the \emph{accuracy} condition, first write
\begin{equation}
    \label{thm2-Qt}
Q_t(s,a) = R(s,a) + \gamma \sum_{s'} T(s,a,s') \max_{a'} Q_t(s',a') + \beta(s,a)
\end{equation}
If $(s,a) \in K_t$, there can be two cases: either $(s,a) \in K_t^1$ or $(s,a) \in K_t^2$.
If $(s,a) \in K_t^1$, then by Definition~\ref{definition:beta} $\beta(s,a) \le 3 \epsilon_1$.
If $(s,a) \in K_t^2$, 
then Lemma~\ref{lubk2} (right-hand side inequality) implies that with probability at least $1-\frac{3\delta}{8}$
\begin{equation}
    \label{accuracy-in1}
     2\epsilon_2 \ge Q_t(s,a) - Q^{*}_{M_{K^{2}_t}}(s,a)
\end{equation}
Meanwhile,
\begin{equation}
    \label{thm2-Q*}
 Q^{*}_{M_{K^{2}_{t}}}(s,a) = R(s,a) + \gamma \sum_{s'} T(s,a,s') \max_{a'} Q^{*}_{M_{K^{2}_{t}}}(s',a')
\end{equation}
and substituting from \eqref{thm2-Q*} and \eqref{thm2-Qt} into \eqref{accuracy-in1} yields
\begin{equation} \label{accuracy-in2}
\gamma \sum_{s'} T(s,a,s') \left( \max_{a'} Q_t(s',a') - \max_{a'} Q^{*}_{M_{K^{2}_{t}}}(s',a')  \right)
    + \beta(s,a) \le 2\epsilon_2
\end{equation}
Let $ a_1 \coloneq \arg\max_{a'} Q_{M_{K^{2}_{t}}}(s',a')$ and bound the difference
\begin{align*}
 \max_{a'} Q_t(s',a') - \max_{a'} Q^{*}_{M_{K^{2}_{t}}}(s',a') 
 &= \max_{a'} Q_t(s',a') - Q^{*}_{M_{K^{2}_{t}}}(s',a_1) \\
 &\ge Q_t(s',a_1) - Q^{*}_{M_{K^{2}_{t}}}(s',a_1)
\end{align*}
Apply Lemma~\ref{lubk2} (left-hand side inequality) to the latter expression to get
\[
 \max_{a'} Q_t(s',a') - \max_{a'} Q^{*}_{M_{K^{2}_{t}}}(s',a') \ge -2\epsilon_2
\]
which implies for \eqref{accuracy-in2} that
\[
  2\epsilon_2 \ge \beta(s,a) -2\gamma\epsilon_2 \implies 
  \beta(s,a) \le 2(1+\gamma)\epsilon_2 \le 3 \epsilon_2
\]
Thus in any case when $(s,a) \in K_t$, $\beta(s,a) \le 3 \epsilon_1$ with probability at least $1-\frac{3\delta}{8}$.
In light of this, considering a policy dictating actions $a = \pi_t(s)$ and mirroring \eqref{thm2-Qt}--\eqref{thm2-Q*} we write for the values of states in which $\big(s,\pi_t(s)\big) \in K_t$
\begin{align*}
v^{\pi_t}_{M_{K_{t}}}(s) &= R\big(s,\pi_t(s)\big)+ \gamma \sum_{s'} T\big(s,\pi_t(s),s'\big) v^{\pi_t}_{M_{K_{t}}}(s') \\
v_t(s) &= R\big(s,\pi_t(s)\big)+ 
 \gamma \sum_{s'} T\big(s,\pi_t(s),s'\big) v_t(s') + \beta(s,a)
\end{align*}
while for those in which $\big(s,\pi_t(s)\big) \notin K_t$, we already know that
\begin{align*}
 v^{\pi_t}_{M_{K_{t}}}(s) &= Q_t\big(s,\pi_t(s)\big) \\
 v_t(s) &= Q_t\big(s,\pi_t(s)\big)
\end{align*}
So now if one denotes
\[
  \alpha \coloneq \max_{s} \big( v_t(s) - v^{\pi_t}_{M_{K_t}}(s) \big) =  v_t(s^\ast) - v^{\pi_t}_{M_{K_t}}(s^\ast) 
\]
then either $\alpha = 0$ (when $\big(s,\pi_t(s)\big) \notin K_t$) or it affords an upper bound
\begin{multline*}
 \gamma \sum_{s'} T(s^{*},\pi_t(s^{*}),s') \big( v_t(s') - v^{\pi_t}_{M_{K_{t}}}(s') \big) +  \beta\big(s^*,\pi_t(s^*)\big) \\ \leq 
    \gamma \sum_{s'} T(s^{*},\pi_t(s^{*}),s') \big( v_t(s') - v^{\pi_t}_{M_{K_{t}}}(s') \big) + 3\epsilon_1 \leq 
    \gamma \alpha +3\epsilon_1 
\end{multline*}
from which it follows that
$
\alpha \le \gamma \alpha + 3 \epsilon_1 \implies 
\alpha \le \tfrac{ 3\epsilon }{ 1 - \gamma } = \epsilon
$.

Finally, to analyze \emph{complexity} invoke Lemmas~\ref{sbound} and \ref{ebound} to see that the learning complexity $\zeta (\epsilon,\delta)$ is bounded by $\kappa + \min{(2 m_1 \kappa, |S||A| m_2)}$ with probability $1-\frac{\delta}{8}$. 

In conclusion, the conditions of Theorem \ref{ddq pac} are satisfied with probability $1-\delta$ and therefore the \ac{ddq} algorithm is \ac{pac}. Substituting $\zeta (\epsilon,\delta)$ into \eqref{learningcom} completes the proof. 

\end{proof}
\medskip


\section{Numerical Results}
\label{sim}

This section opens with a comparison of the \ac{ddq} algorithm to its parent technologies. 
It proceeds with additional comparisons to the state-of-the-art in both model-based \cite{szita2010} as well as model-free \cite{dong2019q} \ac{rl} algorithms.
For this comparison, the algorithms with the currently best sample complexity are implemented on a type of \ac{mdp} which has been proposed and used in literature as a model which is objectively difficult to learn~\cite{strehl2009}.

The first round of comparisons start with  R-max, Delayed Q-learning and \ac{ddq} being implemented on a small-scale grid-world example (Fig.~\ref{gw}). 
This example test case has nine states, with the initial state being the one labeled $1$, and the terminal (goal) state labeled $9$. 
Each state is assigned a reward of $0$ except for the terminal state which has $1$.
For this example, $\gamma  \coloneqq 0.8$. 
In all states but the terminal one, the system has four primitive actions available: {\small \textsf{down}} ($\mathsf{d}$), {\small \textsf{left}} ($\mathsf{l}$), {\small \textsf{up}} ($\mathsf{u}$), and {\small \textsf{right}} ($\mathsf{r}$). 
The grid-world of Fig.~\ref{gw} includes cells with two types of  boundaries: the boundaries marked with a single-line afford transition probabilities of $0.9$ through them;
the boundaries marked with a double line afford transitions through them at probability $0.1$. 
The optimal policy for this grid-world example is shown in Fig.~\ref{gww}.


\begin{figure}[h!]
\centering
\includegraphics[width=0.25\textwidth]{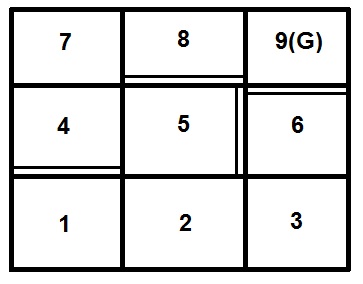}
\caption{The grid-world example.}
\label{gw}
\end{figure}

\begin{figure}[h!]
\centering
\includegraphics[width=0.25\textwidth]{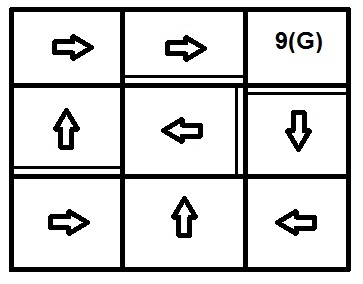}
\caption{The actual optimal policy in the grid-world example.}
\label{gww}
\end{figure}

Initializing the three \ac{pac} algorithms with parameters $m_1=65$, $m_2=175$ and $\varepsilon=0.06$, yields the performance metrics shown in Table~\ref{result}, which are measured in terms of the number of samples needed to reach at $4\varepsilon$ optimality, averaged over $10$ algorithm runs.
Parameters $m_1$ and $m_2$ are intentionally chosen to enable a fair comparison, in the sense that the sample complexity of the model-free Delayed Q-learning, and the model-based R-max algorithms is almost identical. 
In this case, and with these same tuning parameters, \ac{ddq} yields a modest but notable sample complexity improvement.

{\footnotesize
\begin{table}[h]
\centering
\caption{Average \# of samples for reaching $4\varepsilon$ optimality}
\label{result}
\setlength{\tabcolsep}{4pt}
\begin{tabular}{cc} 
\toprule
{\footnotesize Algorithms}    & {\footnotesize \# of samples}   \\
\midrule
\addlinespace
Delayed Q-learning                 & 
{\scriptsize $6622$} 
\\ \addlinespace
R-max                 & 
{\scriptsize $6727$} 
\\ \addlinespace
\ac{ddq}                 & 
{\scriptsize $5960$} 
\\ \addlinespace
\bottomrule
\end{tabular}
\end{table}
}

The lowest known bound on the sample complexity of a model-based \ac{rl} algorithm on a infinite-horizon \ac{mdp} is $\frac{|S||A|}{\epsilon^2{(1-\gamma)}^6}$ (by the Mormax algorithm~\cite{szita2010}). 
For the model-free case (again on a infinite-horizon \ac{mdp}), the lowest bound on the sample complexity is $\frac{|S||A|}{\epsilon^2{(1-\gamma)}^7}$, achieved by UCB Q-learning~\cite{dong2019q} (the extended version of \cite{jin2018q} which is for finite-horizon \ac{mdp}). 

To perform a fair and meaningful comparison of these algorithms to \ac{ddq}, consider a family of ``difficult-to-learn'' \ac{mdp} as Fig.~\ref{hmdp}. 
The \ac{mdp} has $N+2$ states as $S=\{1,2,…,N,+,- \},$ and $A$ different actions. Transitions from each state $i \in \{1,…,N\}$ are the same, so only transitions from state $1$ are shown.  
One of the actions (marked by solid line) deterministically transports the agent to state $+$ with reward $0.5+\epsilon'$ (with $\epsilon'>0$). 
Let $a$ be any of the other $A-1$ actions (represented by dashed lines). 
From any state $i \in \{1,…,N\}$, taking action $a$ will trigger a transition to state $+$ with reward $1$ and probability $p_{ia}$, or to state $-$ with reward $0$ and probability $1-p_{ia}$, where $p_{ia} \in \{0.5,0.5+2\epsilon'\}$ are numbers very close to $0.5+\epsilon'$. 
For each state $i \in \{1,…,N\}$, there is at most one $a$ such that $p_{ia}=0.5+2\epsilon'$. 
Transitions from states $+$ and $-$ are identical; they simply reset the agent to one of the states $\{1,…,N\}$ uniformly at random. 
\begin{figure}[h!]
\centering
\includegraphics[width=0.15\textwidth]{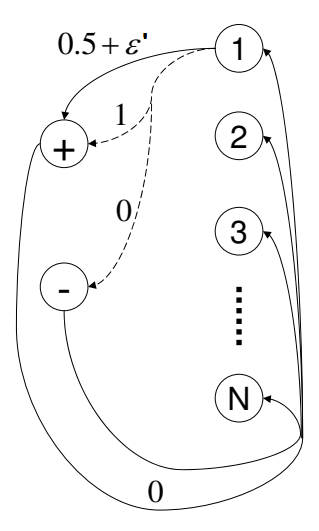}
\caption{A family of difficult-to-learn \ac{mdp}s. \cite{strehl2009}}
\label{hmdp}
\end{figure}

For an \ac{mdp} such as the one shown in Fig.~\ref{hmdp}, the optimal action in any state $i \in \{1,…,N\}$ is independent of the other states; 
specifically, it is the action marked by the solid arrow if $p_{ia}=0.5$ for all dashed actions $a$, or the action marked by the dashed arrow for which $p_{ia}=0.5+2\epsilon'$, otherwise.
Intuitively, this \ac{mdp} is hard to learn for exactly the same reason that a biased coin is hard to be recongized as such if its bias (say, the probability of landing on head) is close to $0.5$ \cite{strehl2009}. 

We thus try to learn such an \ac{mdp} $M$ with $N=2$, $A=2$, and $\epsilon'=0.04$. 
The accuracy that the learned policy should satisfy is set to $\epsilon=0.0025$, and the probability of failure is set to $\delta=0.01$. 
Results are averaged over $50$ runs of each algorithm running on \ac{mdp} $M$.

We empirically fine-tune the parameters of Mormax and UCB Q-learning algorithms to maximize their performance on learning the near optimal ($4\epsilon$-optimal) policy of $M$ in terms of the required samples. 
As expected, the required samples decrease (almost linearly) in $m$ (Fig.~\ref{fig3}) until the necessary condition for the convergence of the algorithm is violated (at around $m=600$).
For that reason, we cap $m$ at $600$ which requires $7770$ samples on average for Mormax to learn the optimal policy.
Yet another important performance metric to record for a model-based \ac{rl} algorithm is the number of times it needs resolve the learned model through value-iteration, since the associated computational effort is highly dependent on this number. 
For Mormax, the average number of times it performs model resolution is $12.06$.     
\begin{figure}[h!]
\centering
\includegraphics[width=0.45\textwidth]{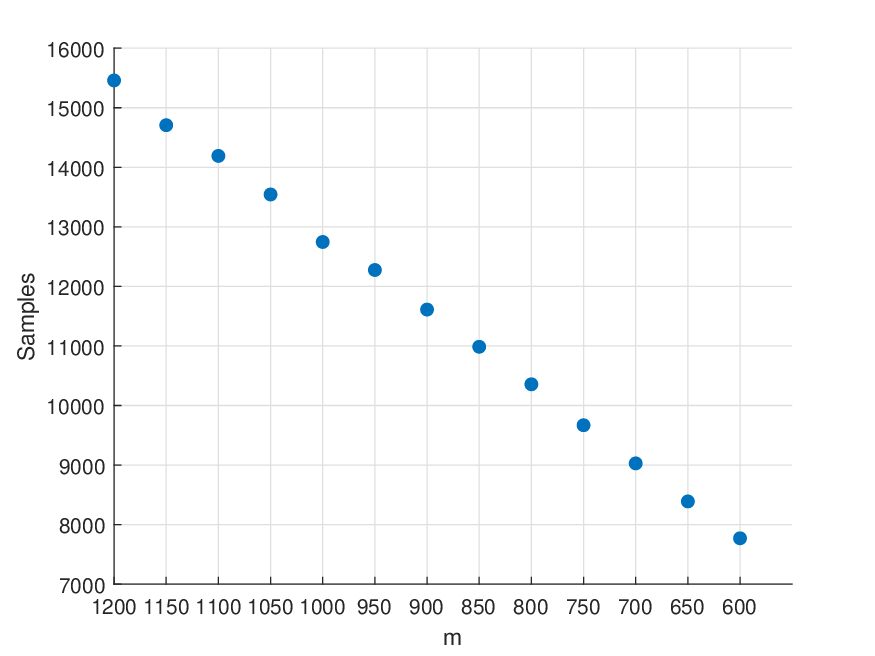}
\caption{The number of samples required by the Mormax algorithm.}
\label{fig3}
\end{figure}

The performance of the UCB Q-learning algorithm appears to be  very sensitive to its $c_2$ parameter.
The value of $4\sqrt{2}$ that has been suggested for $c_2$~\cite{dong2019q} proved very conservative, with the algorithm sometimes requiring millions of data for converging to the optimal policy on $M$. 
The reason is that values of $c_2$ that high cause the effective updates to start when the learning rate has already become very small, thus slowing down the convergence speed. 
We therefore tune the UCB Q-learning algorithm to achieve maximum performance on $M$ by setting its parameter $c_2=\sfrac{1}{50}$ (see Fig.~\ref{fig2}); with this setting, the algorithm requires $8097$ samples to learn the optimal policy on average.  
Setting $c_2<\sfrac{1}{50}$ may cause the algorithm to lie outside the upper confidence interval, and as a result, the algorithm either requires an actual higher number of samples or it fails to convege altogether to the optimal policy after $10^6$ samples.  
\begin{figure}[h!]
\centering
\includegraphics[width=0.45\textwidth]{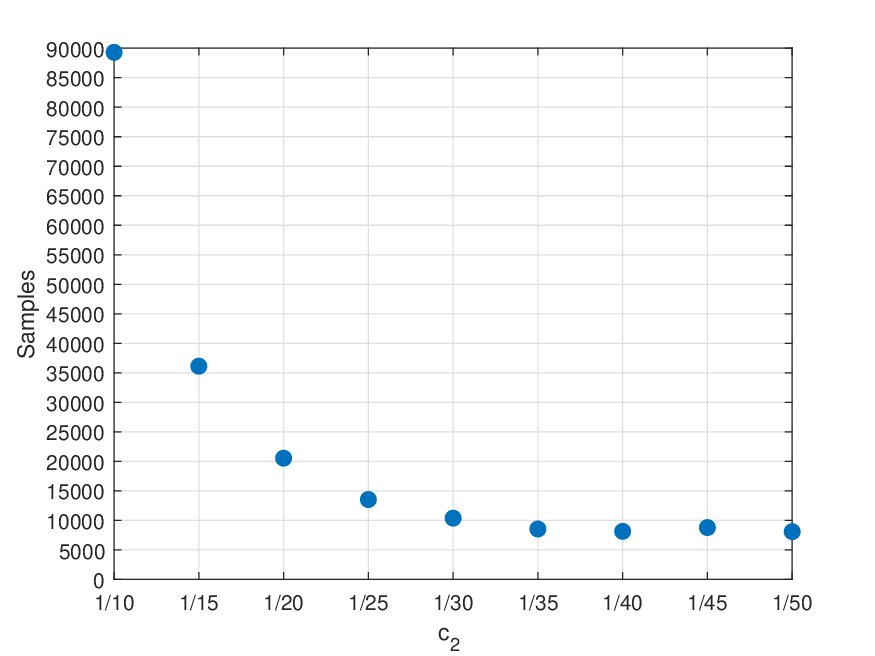}
\caption{The required samples by UCB Q-learning algorithm}
\label{fig2}
\end{figure}

We compare the best performance we could achieve with Mormax and UCB Q-learning with that of \ac{ddq} which we tune with $m_1=150$ and $m_2=750$. 
The average required samples required by \ac{ddq} for learning the $4\epsilon$-optimal policy on $M$ is $5662$, while the number of times that the R-max component of the algorithm resolves the model through value-iteration part is $3.76$ on average.

Thus, although the provable worst-case bound on the sample complexity of \ac{ddq} algorithm appears higher than that of Mormax and UCB Q-learning, \ac{ddq} can outperform both algorithms in terms of the required data samples, especially in difficult learning tasks.
What is more, the hybrid nature of \ac{ddq} algorithm enables significant savings in terms of computational effort ---the latter captured by the number of times when the algorithm resorts to model resolution--- compared to model-based algorithms like Mormax. 
Table~\ref{result2} summarizes the results of this comparison. 

{\footnotesize
\begin{table}[h]
\centering
\caption{The best possible performance on learning \ac{mdp} $M$}
\label{result2}
\setlength{\tabcolsep}{4pt}
\begin{tabular}{ccc} 
\toprule
{\footnotesize Algorithms}    & {\footnotesize \# of samples} &{\footnotesize \# of model resolution}   \\
\midrule
\addlinespace
Mormax                 & 
{\scriptsize $7770$}   &
{\scriptsize $12.06$}
\\ \addlinespace
UCB Q-learning                 & 
{\scriptsize $8097$}           &
{\scriptsize $0$ (model-free)}
\\ \addlinespace
\ac{ddq}                 & 
{\scriptsize $5662$}     &
{\scriptsize $3.76$}
\\ \addlinespace
\bottomrule
\end{tabular}
\end{table}
}
\medskip


\section{Conclusion} \label{section:conclusion}
It is possible to build an \ac{rl} algorithm that captures favorable features of both model-based and model-free learning and most importantly preserves the \ac{pac} property. 
One such algorithm is the \ac{ddq};
this algorithm leverages the idea of Dyna-Q to combine two existing \ac{pac} algorithms, namely the model-based R-max and the model-free Delayed Q-learning, in a way that achieves the best (complexity results) of both. 
Theoretical analysis establishes that \ac{ddq} enjoys a sample complexity that is at worst as high as the smallest of its constituent technologies; yet, in practice, as the numerical example included suggests, \ac{ddq} can outperform them both. 
In addition, numerical example on the comparison of \ac{ddq} to the state of the art in model-based and model free \ac{rl} suggests clear advantages in practical implementations.
\medskip


\bibliographystyle{unsrt}  
\bibliography{My_Collection} 


\appendix
\appendixpage

\section{Proof of Lemma~\ref{sbound}}
\label{L1}
Consider a fixed state-action pair $(s,a)$. 
Its value $Q(s,a)$ is initially set to $v_\mathrm{max}=\frac{1}{1-\gamma}$.
When an update of type-$1$ (Algorithm~\ref{alg:ddq} line 30) is successful $Q(s,a)$ is reduced by at least $\epsilon_1$. 
Since the reward function $R(s,a)$ is non-negative, we must have $Q(s,a) \geq 0$ in all timesteps, which means that there can be at most $\frac{1}{\epsilon_1(1-\gamma)}$ updates of type-$1$ for $(s,a)$.
On the other hand, a type-2 update  (Algorithm~\ref{alg:ddq} line 51) can  occur only once when $n(s,a)=m_2$. 
Therefore, the total number of successful timesteps for $(s,a)$ is at most $1+\frac{1}{\epsilon_1(1-\gamma)}$ times. 
With $|S| |A|$ total state-action pairs, the total number of successful timesteps is bounded by $\kappa = |S||A|+\frac{|S||A|}{(1-\gamma)\epsilon_1}$.

\section{Proof of Lemma~\ref{ubound}}
\label{L2} 
Suppose an attempted update  occurs at timestep $t$ to some $(s,a)$. 
By definition, for a subsequent attempted update to $(s,a)$ to occur at timestep $t' > t$, at least one successful timestep must occur between $t$ and $t'$.
Lemma \ref{sbound} ensures that there can be no more than $\kappa$ successful timesteps.
In other words, the most frequent occurrence of attempted updates is interlaced between successful updates, which implies that at most $1+\kappa$ attempted updates are possible for $(s,a)$. 
Scaling this argument to all  state-action pairs we arrive at the $|S||A|(1+\kappa)$ upper bound.

\section{Proof of Lemma~\ref{a1}}
\label{L3}
Let $ Q^{*}_{M_{K_t}}(s^{*},a^{*})$ denote $\max_{(s,a)} Q^{*}_{M_{K_t}}(s,a)$. If $(s^*,a^*) \notin K_t$, we are done since $Q^{*}_{M_{K_t}}(s^{*},a^{*})=Q_t(s^*,a^*) \leq \frac{1}{1-\gamma}$. Otherwise, for $(s^{*},a^{*}) \in K_t$  write
\begin{align*}
   Q^{*}_{M_{K_t}}(s^{*},a^{*}) &= 
   R(s^{*},a^{*})+ \gamma \sum_s T(s^{*},a^{*},s) \max_a Q^{*}_{M_{K_t}}(s,a) \\ 
   &\leq  R(s^{*},a^{*})+ \gamma \,Q^{*}_{M_{K_t}}(s^{*},a^{*})\, \sum_s T(s^{*},a^{*},s) \\
   &=  R(s^{*},a^{*})+ \gamma \,Q^{*}_{M_{K_t}}(s^{*},a^{*}) \\
   & \leq 
   1+ \gamma Q^{*}_{M_{K_t}}(s^{*},a^{*}) \implies 
   Q^{*}_{M_{K_t}}(s^{*},a^{*})
   \leq  \frac{1}{1-\gamma} 
\end{align*}

\section{Proof of Lemma~\ref{type1opt}}
\label{L4}
Let an update of type-$1$ occur for $(s,a)$ at timestep $t$. 
Suppose that the latest $m_1$ experiences of $(s,a)$ happened at timesteps $t_1 <t_2< \cdots <t_{m_1}=t$, when the system was rewarded 
$r[1], r[2], \ldots, r[m_1]$ and jumped to states $s[1], s[2], \ldots, s[m_1]$, respectively. 
Define the random variable $Y=r[i]+\gamma v^{*}_M(s[i])$ for $1 \leq i \leq m_1$ and note that $0 \leq Y \leq \frac{1}{1-\gamma}$. 
Then a direct application of the Hoeffding inequality for bounded random variables and with the choice of $m_1$ as in \eqref{m1} implies that
\begin{equation*}
    \frac{1}{m_1} \sum^{m_1}_{i=1} \big( r[i]+\gamma v^{*}_M(s[i]) \big) >
     \mathbb{E} \big\{Y\big\} - (\epsilon_1-2\epsilon_2) = Q^{*}_M(s,a)-\epsilon_1+2\epsilon_2
\end{equation*}
with probability $1-\sfrac{\delta}{8\big(|S||A|(1+\kappa)\big)}$.

Now we have:
\begin{align*}
    Q'(s,a)& =\frac{1}{m_1} \big( \sum^{m_1}_{i=1} r[i]+\gamma v_{t_i}(s[i]) \big) + \epsilon_1 
    \\ &\geq \frac{1}{m_1} \big( \sum^{m_1}_{i=1} r[i]+\gamma v^{*}_M(s[i]) \big) - 2\gamma\epsilon_2 +\epsilon_1 
    \\ &\geq Q^{*}_M(s,a) - \epsilon_1 + 2\epsilon_2 - 2\gamma\epsilon_2 + \epsilon_1 \geq Q^{*}_M(s,a)
\end{align*}
Finally, we want this fact to be true for all possible attempted updates of type-$1$. According to Lemma~\ref{ubound}, an upper bound for all possible attempted updates is $|S||A|(1+\kappa)$. Therefore, the above fact is true with probability at least  $\big(1-\sfrac{\delta}{8\big(|S||A|(1+\kappa)\big)}\big)^{|S||A|(1+\kappa)}$. 
An induction argument can now be employed to show that $1-\frac{\delta}{8}$ bounds the above expression from below.

\section{Proof of Lemma~\ref{a2}}
\label{L6}
First note that $K^{2}_{t_1} \subseteq K^{2}_{t_2}$. For all $(s,a) \notin K^{2}_{t_2}$
\begin{equation} \label{t1t2}
     Q^{*}_{M_{K^{2}_{t_1}}}(s,a) = Q_{t_1}(s,a) \geq Q_{t_2}(s,a) = Q^{*}_{M_{K^{2}_{t_2}}}(s,a)
\end{equation}
while for all $(s,a) \in K^{2}_{t_1}$
\begin{align*}
Q^{*}_{M_{K^{2}_{t_1}}}\!\!\!(s,a) &= R(s,a)+\gamma \sum_{s^{'}} T(s,a,s') \max_{a'} Q^{*}_{M_{K^{2}_{t_1}}}\!\!\!(s',a')
\\
Q^{*}_{M_{K^{2}_{t_2}}}\!\!\!(s,a) &= R(s,a)+\gamma \sum_{s^{'}} T(s,a,s') \max_{a'} Q^{*}_{M_{K^{2}_{t_2}}}\!\!\!(s',a')
\end{align*}
implying
\begin{equation}
\label{eqa}
    Q^{*}_{M_{K^{2}_{t_1}}}\!\!\!(s,a) - Q^{*}_{M_{K^{2}_{t_2}}}\!\!\!(s,a) = \gamma \sum_{s^{'}} T(s,a,s') 
    \times \big( \max_{a'} Q^{*}_{M_{K^{2}_{t_1}}}\!\!\!(s',a') - \max_{a'} Q^{*}_{M_{K^{2}_{t_2}}}\!\!\!(s',a')  \big)
\end{equation}
Every $(s,a) \in K^{2}_{t_2} \setminus K^{2}_{t_1}$ falls in one of the following categories:
\begin{itemize}[leftmargin=*]
    \item $(s,a$) is a state-action pair that has not been updated ever before or at timestep $t_1$. The Lemma~\ref{a1} implies
    \begin{equation*}
        Q^{*}_{M_{K^{2}_{t_1}}}\!\!\!(s,a)=Q_{t_1}\!\!(s,a)=v_\mathrm{max}=\frac{1}{1-\gamma} \geq Q^{*}_{M_{K^{2}_{t_2}}}\!\!\!(s,a)
    \end{equation*}
    which completes the proof.
    \item $(s,a)$ is a state-action pair that has experienced an type-$1$ update before or at $t_1$. 
    Assume that the most recent type-$1$ update of $(s,a)$ occurred at some timestep $t \leq t_1$. 
    Suppose that the $m_1$ visits to $(s,a)$ that triggered this update occurred at instances $t^1 < t^2 < ... < t^{m_1}=t \leq t_1$, and the observed rewards and next states were $r[1],r[2],...,r[m_1]$ and $s[1],s[2],...,s[m_1]$, respectively. 
    For the random variable $Z=r[i]+\gamma v_{t}(s[i])$,
    \[
        \mathbb{E}\{Z\}=R(s,a)+\gamma \sum_{s'} T(s,a,s') \max_{a'} Q_{t}(s',a')
    \]
    Then
    \begin{equation*}
    Q^{*}_{M_{K^{2}_{t_1}}}(s,a)=Q_{t_1}(s,a)=Q_t(s,a) =
       \frac{\sum^{m_1}_{i=1} r[i]+\gamma v_{t^i}(s[i])}{m_1} + \epsilon_1 
       \ge 
       \frac{\sum^{m_1}_{i=1} r[i]+\gamma v_{t}(s[i])}{m_1} + \epsilon_1
    \end{equation*}
    and applying Hoeffding inequality to the right hand side
    \begin{align*}
    Q^{*}_{M_{K^{2}_{t_1}}}(s,a) 
        > \mathbb{E}\{Z\}-\epsilon_1+2\epsilon_2+\epsilon_1  
       &= R(s,a)+\gamma \sum_{s'} T(s,a,s') \max_{a'} Q_{t}(s',a') +2\epsilon_2  \\
       &\geq R(s,a)+\gamma \sum_{s^{'}} T(s,a,s') \max_{a'} Q_{t_1}\!(s',a') +2\epsilon_2  \\
       &\stackrel{\eqref{t1t2}}{\geq} R(s,a)+\gamma \sum_{s'} T(s,a,s') \max_{a'} Q^{*}_{M_{K^{2}_{t_1}}}\!\!\!(s',a')
    \end{align*}
    with probability $1-\tfrac{\delta}{8|S||A|(1+\kappa)}$. 
    Then --- following the final steps of Lemma~\ref{type1opt} --- with probability at least $1-\frac{\delta}{8}$
     after all possible attempted updates,
    \begin{equation}
    \label{eqb}
    Q^{*}_{M_{K^{2}_{t_1}}}(s,a) - Q^{*}_{M_{K^{2}_{t_2}}}(s,a) \geq 
    \gamma \sum_{s^{'}} T(s,a,s') \big( \max_{a'} Q^{*}_{M_{K^{2}_{t_1}}}(s',a') - \max_{a'} Q^{*}_{M_{K^{2}_{t_2}}}(s',a')  \big)
\end{equation}
\end{itemize}
In any case, therefore, i.e., either when $(s,a) \notin K_{t_2}^2$ or when $(s,a) \in K_{t_2}^2 \setminus K_{t_1}^2$, one can define
\begin{align*}
 \alpha &\coloneq \min_{(s,a)} \big( Q^{*}_{M_{K^{2}_{t_1}}}(s,a) - Q^{*}_{M_{K^{2}_{t_2}}}(s,a) \big)  \coloneq 
 Q^{*}_{M_{K^{2}_{t_1}}}(s^{*},a^{*}) - Q^{*}_{M_{K^{2}_{t_2}}}(s^{*},a^{*})   
\end{align*}
and if $\alpha \geq 0$ recognize that the proof is completed. 
Assume for the sake of argument that $\alpha < 0$; then still either \eqref{eqa} is true if $(s,a) \notin K_{t_2}^2$, or \eqref{eqb} if $(s,a) \in K_{t_2}^2 \setminus K_{t_1}^2$. 
Let $ a_{s'} \coloneq \arg\max_{a'} Q^{*}_{M_{K^{2}_{t_2}}}(s',a')$, then in either case,
\begin{align*}
    \alpha &= 
     Q^{*}_{M_{K^{2}_{t_1}}}(s^{*},a^{*}) - Q^{*}_{M_{K^{2}_{t_2}}}(s^{*},a^{*}) 
    \\
    & \geq \!
    \gamma \!\sum_{s'} T(s^{*}\!,a^{*}\!,s') \big(\! \max_{a'} Q^{*}_{M_{K^{2}_{t_1}}}\!\!(s',a')\! - \! \max_{a'} Q^{*}_{M_{K^{2}_{t_2}}}\!\!(s',a') \! \big) \\ 
    & = \!
    \gamma \!\sum_{s'} T(s^{*}\!,a^{*}\!,s') \big(\! \max_{a'} Q^{*}_{M_{K^{2}_{t_1}}}\!\!(s',a')\! - \! Q^{*}_{M_{K^{2}_{t_2}}}\!\!(s',a_{s'}) \! \big) \\
     & \geq \!
    \gamma \!\sum_{s'} T(s^{*}\!,a^{*}\!,s') \big(\! Q^{*}_{M_{K^{2}_{t_1}}}\!\!(s',a_{s'})\! - \! Q^{*}_{M_{K^{2}_{t_2}}}\!\!(s',a_{s'}) \! \big) \\ &\geq 
    \gamma \alpha \implies \alpha \geq 0
\end{align*}
which is a contradiction.  
Therefore $\alpha$ cannot be negative and therefore $Q^{*}_{M_{K^{2}_{t_1}}}(s,a) - Q^{*}_{M_{K^{2}_{t_2}}}(s,a) \ge 0$.

\section{Proof of Lemma~\ref{lubk2}}
\label{L8}
For all $(s,a) \notin K^{2}_t$
\begin{subequations}\label{lm8fst}
\begin{equation}
    \label{lm8s1}
Q_t(s,a) = Q^{*}_{M_{K^{2}_t}}(s,a) \leq Q^{*}_{M_{K^{2}_t}}(s,a) + 2\epsilon_2
\end{equation}
Now for $(s,a) \in K^{2}_t$, and referring to line 50 of Algorithm~\ref{alg:ddq} one sees that for timestep $t$ it is $Q_t(s,a) \leq Q_\mathrm{vl}(s,a)$.  
Meanwhile, for timestep $t$ Lemma~\ref{vi} ensures
\begin{equation}
    \label{lm8s2}
Q_\mathrm{vl}(s,a) \leq Q^{*}_{\hat{M}_{K^{2}_t}}(s,a) + \epsilon_2
\end{equation}
while Lemma~\ref{estimation} implies
\begin{equation}
    \label{lm8s3}
    Q^{*}_{\hat{M}_{K^{2}_t}}(s,a) + \epsilon_2 \leq Q^{*}_{M_{K^{2}_t}}(s,a) + 2\epsilon_2 
\end{equation}
\end{subequations}
with probability $1-\frac{\delta}{8}$. 
Combining \eqref{lm8fst} one obtains the right hand side of \eqref{lm8}.
Establishing the left hand side of \eqref{lm8} is 
done by strong induction. 
At $t=1$, we have $K^{2}_1 = \emptyset$ and thus
\begin{equation*}
Q_1(s,a) = Q^{*}_{M_{K^{2}_1}}(s,a) \geq Q^{*}_{M_{K^{2}_1}}(s,a) - 2\epsilon_2
\end{equation*}
Assume that $Q_{t}(s,a) = Q^{*}_{M_{K^{2}_{t}}}(s,a) \geq Q^{*}_{M_{K^{2}_{t}}}(s,a) - 2\epsilon_2$ for $t \le n-1$. 
If timestep $t=n$ is not a successful timestep (Definition~\ref{successful-t}), nothing happens so equality holds; thus let us assume that $t=n$ is successful.
Then, and for all $(s,a) \notin K^{2}_n$ we have automatically
\begin{equation*}
Q_n(s,a) = Q^{*}_{M_{K^{2}_n}}(s,a) \geq Q^{*}_{M_{K^{2}_n}}(s,a) - 2\epsilon_2
\end{equation*}
Just as before, for $(s,a) \in K^{2}_n$ for which a type-$2$ update succeeded at timestep $t$
\begin{subequations}\label{lm8snd}
\begin{equation}
Q_n(s,a) = Q_\mathrm{vl}(s,a) \geq Q^{*}_{\hat{M}_{K^{2}_n}}(s,a) - \epsilon_2 
\end{equation}
as a result of Lemma~\ref{vi}, and
\begin{equation}
Q^{*}_{\hat{M}_{K^{2}_n}}(s,a) - \epsilon_2 \geq      Q^{*}_{M_{K^{2}_n}}(s,a) - 2\epsilon_2 
\end{equation}
with probability $1-\frac{\delta}{8}$, due to Lemma~\ref{estimation}. 
\end{subequations}
For those $(s,a) \in K^{2}_n$ for which a type-$2$ update did \emph{not} succeed at timestep $t$, it is $Q_n(s,a)=Q_{n-1}(s,a)$ and there are three distinct possibilities:
\begin{itemize}[leftmargin=*]
    \item Value $Q_{n-1}(s,a)$ has never been updated before. 
    Then,
    \begin{equation*}
        Q_n(s,a)=\tfrac{1}{1-\gamma} \stackrel{\text{Lemma}~\ref{a1}}{\geq} Q^{*}_{M_{K^{2}_n}}(s,a) \geq Q^{*}_{M_{K^{2}_n}}(s,a) - 2\epsilon_2 
    \end{equation*}
    \item The most recent update for $(s,a)$ was of type-$2$ and occured at some $t \leq n-1$. 
    Then,
    \begin{equation*}
    Q_n(s,a) \stackrel{\text{Lemmas~ \ref{vi}\&\ref{estimation}}}{\geq} Q^{*}_{M_{K^{2}_t}}(s,a) - 2\epsilon_2 
    \end{equation*}
    with probability $1-\tfrac{\delta}{8}$, and
    \begin{equation*}
    Q^{*}_{M_{K^{2}_t}}(s,a) - 2\epsilon_2 
    \stackrel{\text{Lemma}~\ref{a2}}{\geq}
    Q^{*}_{M_{K^{2}_n}}(s,a) - 2\epsilon_2 
    \end{equation*}
    also with with probability $1-\tfrac{\delta}{8}$, so
    \[
    Q_n(s,a) \ge Q^{*}_{M_{K^{2}_n}}(s,a) - 2\epsilon_2
    \]
    with probability at least $1-\tfrac{2\delta}{8} \le (1 - \tfrac{\delta}{8})^2$.       
    \item The most recent update for $(s,a)$ was of type-$1$ and occured at some $t' \leq n-1$. 
    Then suppose that the $m_1$ collection of visits of $(s,a)$ for this update occurred at timesteps $t^1 < t^2 < \cdots < t^{m_1}=t^{'} \leq n-1$, with the corresponding observed reward and next states being $r[1],r[2],\ldots,r[m_1]$ and $s[1],s[2],\ldots,s[m_1]$, respectively. 
    The expectation of the random variable $F=r[i]+\gamma v_{t^{m_1}}(s[i])$ is
    \begin{equation*}
        \mathbb{E}\{F\}=R(s,a)+\gamma \sum_{s'} T(s,a,s') \max_{a'} Q_{t^{m_1}}(s',a')
    \end{equation*}
    which, with the use of Hoeffding inequality, bounds the sum in
    \begin{align*}
       Q_n(s,a)&= \frac{1}{m_1} \left(\sum^{m_1}_{i=1} r[i]+\gamma v_{t^i}(s[i]) \right) + \epsilon_1  \\
       &\geq \frac{1}{m_1} \left(\sum^{m_1}_{i=1} r[i]+\gamma v_{t^{m_1}}(s[i]) \right) + \epsilon_1 > \mathbb{E}\{F\}-\epsilon_1+2\epsilon_2+\epsilon_1 \\
       &= R(s,a)+\gamma \sum_{s'} T(s,a,s') \max_{a'} Q_{t'}(s',a') +2\epsilon_2 
    \end{align*}
    and yields
    \[
    Q_n(s,a) \ge 
    R(s,a)+\gamma \sum_{s'} T(s,a,s') \max_{a'} Q_{n}(s',a')
    \]
    with probability $1-\sfrac{\delta}{8\big(|S||A|(1+\kappa)\big)}$. 
    Following the steps in the proof of Lemma~\ref{type1opt} when thinking of  all possible attempted updates, one states the above with probability $1-\frac{\delta}{8}$.
    Subtracting now $Q^{*}_{M_{K^{2}_{n}}}(s,a)$ from both sides yields
    \begin{equation}
    \label{type13}
    \gamma \sum_{s'} T(s,a,s') \left( \max_{a'} Q_n(s',a') - \max_{a'} Q^{*}_{M_{K^{2}_{n}}}(s',a')  \right)
    \le     Q_n(s,a) - Q^{*}_{M_{K^{2}_{n}}}(s,a)
    \end{equation}
\end{itemize}
and if one denotes
\begin{equation*}
 \alpha \coloneq \min_{(s,a)} \big( Q_n(s,a) - Q^{*}_{M_{K^{2}_{n}}}(s,a) \big)
 = Q_n(s^{*},a^{*}) - Q^{*}_{M_{K^{2}_{n}}}(s^{*},a^{*})   
\end{equation*}
then we want to show $\alpha \geq -2\epsilon_2$. Let $ a_{s'} \coloneq \arg\max_{a'} Q^{*}_{M_{K^{2}_{n}}}(s',a')$, then \eqref{type13} implies
\begin{align*}
    \alpha &= 
     Q_n(s^{*},a^{*}) - Q^{*}_{M_{K^{2}_{n}}}(s^{*},a^{*}) 
    \\
    & \geq \!
    \gamma \!\sum_{s'} T(s^{*}\!,a^{*}\!,s') \big(\! \max_{a'} Q_n\!\!(s',a')\! - \! \max_{a'} Q^{*}_{M_{K^{2}_{n}}}\!\!(s',a') \! \big) \\ 
    & = \!
    \gamma \!\sum_{s'} T(s^{*}\!,a^{*}\!,s') \big(\! \max_{a'} Q_n\!\!(s',a')\! - \! Q^{*}_{M_{K^{2}_{n}}}\!\!(s',a_{s'}) \! \big) \\
     & \geq \!
    \gamma \!\sum_{s'} T(s^{*}\!,a^{*}\!,s') \big(\! Q_n\!\!(s',a_{s'})\! - \! Q^{*}_{M_{K^{2}_{n}}}\!\!(s',a_{s'}) \! \big) \\ &\geq 
    \gamma \alpha \implies \alpha \geq 0 \geq -2\epsilon_2
\end{align*}

Summing up, the right side of \eqref{lm8} holds with probability $1-\frac{\delta}{8}$, while the left side is true with probability at least $(1-\frac{\delta}{12})^2$.
Together, both inequalities are true with probability at least $(1-\frac{\delta}{12})^3 \ge 1 -\frac{3\delta}{8}$.

\section{Proof of Lemma~\ref{lemma6}}
\label{L9}
Assume that at timestep $t$, $(s,a) \notin K_t$, $l(s,a)=0$ and $\mathsf{learn}(s,a)=\mathrm{true}$, and suppose that $m_1$ experiences of $(s,a)$ after $t$ happen at timesteps $t \leq t_1<t_2<\cdots<t_{m_1}$. Let $r[1],r[2],\ldots,r[m_1]$ and $s[1],s[2],\ldots,s[m_1]$ be the rewards and next states observed for the $m_1$ experiences of $(s,a)$.
Then define the random variable $X = r[i]+ \gamma v_{t_1}(s[i])$ letting $i$ range in $\{1,\ldots, m_1\}$, and note that $0 \leq X \leq 1$.

A direct application of the Hoeffding inequality with the choice of $m_1$ as in \eqref{m1} yields
\begin{equation*}
    \frac{1}{m_1} \left(\sum^{m_1}_{i=1} r[i] + \gamma v_{t_1}(s[i]) \right) - \mathbb{E} \big\{X\big\} < \epsilon_1-2\epsilon_2 < \epsilon_1
\end{equation*}
with probability $1- \frac{\delta}{8|S||A|(1+\kappa)}$.
Since the \ac{ddq} algorithm only allows for updates that decrease the value estimate for any stat-action pairs, we can write:
\begin{align*}
    \label{lem6}
    Q_{t}(s,a)- \frac{1}{m_1} \left(\sum^{m_1}_{i=1} r[i] + \gamma v_{t_i}(s[i]) \right) 
    &\geq Q_t(s,a)- \frac{1}{m_1} \left(\sum^{m_1}_{i=1} r[i] + \gamma v_{t_1}(s[i]) \right) 
    \\ &> Q_t(s,a) - \mathbb{E} \big\{X\big\} - \epsilon_1 
\end{align*}
and because $(s,a) \notin K_t$ meaning $Q_t(s,a)-\mathbb{E} \big\{X\big\} > 3\epsilon_1$,
\[    Q_t(s,a) - \mathbb{E} \big\{X\big\} - \epsilon_1 > 2\epsilon_1
\]
guaranteeing success for the type-1 update at timestep $t_{m_1}$.
Since for the case that $l(s,a)=0$ and $\mathsf{learn}(s,a)=\mathrm{true}$, an attempted update will necessarily happen; there can be at most $|S||A|(1+\kappa)$ instances of such an event. 
Working in a fashion similar to the proof of Lemma~\ref{type1opt}, one concludes that the lemma's statement holds with probability at least $1-\frac{\delta}{8}$.

\section{Proof of Lemma~\ref{lemma7}}
\label{L10}
We will assume that $(s,a)$ has not already been visited $m_2$ times before timestep $t$, because then it is obvious that $(s,a) \in K_{t+1}$.
Thus we work under the assumption that $(s,a)$ has been visited fewer than $m_2$ times up until $t$, at which time an unsuccessful update of $(s,a)$ occurs, while right after at $t+1$ we see $\mathsf{learn}(s,a)=\mathrm{false}$.
Now set up a contradiction argument: under those conditions, \emph{assume that $(s,a) \notin K_{t+1}$.}
Since the update at $t$ was unsuccessful, $K_{t+1} = K_t$, which would also imply that $(s,a) \notin K_t$.
Now label the times of the most recent $m_1$ experiences of $(s,a)$ as $b(s,a) \triangleq t_1<t_2<\cdots<t_{m_1}=t$.
The contrapositive of the statement proved in Lemma~\ref{lemma6}, suggests that since the update at $t$ is unsuccessful, it must be $(s,a) \in K_{t_1}$.
Since $(s,a) \notin K_t$, some timestep between $t_1$ and $t$ must have been successful.
Let us denote that successful timestep $t^\ast > b(s,a)$.
But then the condition $t_1 = b(s,a) < t^\ast$ would not allow the $\mathsf{learn}$ flag to be set to $\mathrm{false}$ in between these two timesteps, and we know from the statement of the lemma that this is true.
Therefore, we have a contradiction; the assumption made is invalid, and therefore $(s,a) \in K_{t} = K_{t+1}$.

\section{Proof of Lemma~\ref{ebound}}
\label{L11}
Fix a state-action pair $(s,a)$, We begin by showing that if $(s,a) \notin K_t$  at timestep $t$, then within at most $2m_1$ more experiences of $(s,a)$ after $t$, a successful timestep for $(s,a)$ must occur.
Toward that end, we analyse the worst case where $m_2$-th visit of $(s,a)$ will not occur within $2m_1$ more experiences of $(s,a)$ after timestep $t$. For $(s,a) \notin K_t$, distinguish two possible cases at the beginning of timestep $t$: either $\mathsf{learn}(s,a) = \mathrm{false}$ or $\mathsf{learn}(s,a) = \mathrm{true}$.
Consider first the case where $\mathsf{learn}(s,a) = \mathrm{false}$. Assume that the most recent attempted update of $(s,a)$ occurred at some timestep $t'$ which was unsuccessful and set the flag $\mathsf{learn}(s,a)$ to $\mathrm{false}$. Then, according to Lemma~\ref{lemma7}, it will be $(s,a) \in K_{t'+1}$.
However, now it is $(s,a) \notin K_t$, which implies that a successful timestep must have occurred at some $t^\ast$ with $t'+1<t^\ast<t$. Thus the flag $\mathsf{learn}(s,a)$ will set to $\mathrm{true}$ during timestep $t$. Then, at $t$ we have all conditions of Lemma~\ref{lemma6} (i.e. $\mathsf{learn}(s,a)=\mathrm{true}$, $(s,a) \notin K_{t}$ and $l(s,a)=0$) and thus the type-1 update upon $m_1$-th visit of $(s,a)$ after $t$ will be successful. 

Take now the case where $\mathsf{learn}(s,a) = \mathrm{true}$.
We know that an attempted type-1 update for $(s,a)$ will occur in at most $m_1$ experiences of $(s,a)$, and those are assumed occurring at timesteps $t_1<\cdots<t_{m_1}$, then $t_1\le t \le t_{m_1}$.
Consider the two possibilities: $(s,a) \notin K_{t_1}$ or $(s,a) \in K_{t_1}$.
In the former case, Lemma~\ref{lemma6} indicates that the attempted update type-1 at $t_{m_1}$ will be successful.
In the latter case, given that $(s,a) \notin K_t$, a successful timestep $t^\ast$ must have taken place between $t_1$ and $t$ (since $K_{t_1}\neq K_t$). Thus, however the attempted update at $t_{m_1}$ is unsuccessful, $\mathsf{learn}(s,a)$ will remain $\mathrm{true}$ and at timestep $t_{m_1}+1$ we will have $\mathsf{learn}(s,a)=\mathrm{true}$, $l(s,a)=0$, and $(s,a)\notin K_{t_{m_1}+1}$; this would trigger Lemma~\ref{lemma6}, and the attempted update type-1 upon $m_1$-th visit of $(s,a)$ after timestep $t_{m_1}+1$ (which is within at most $2m_1$ more experiences of $(s,a)$ after $t$), will be successful. 

Thus far, we showed that after $(s,a) \notin K_t$, within at most $2m_1$ more experiences of $(s,a)$, at least one successful timestep for $(s,a)$ must occur. According to lemma \ref{sbound}, the total number of successful timesteps for $(s,a)$ are bounded by $1+\frac{1}{(1-\gamma)\epsilon_1}$. This means that the total number of timesteps with $(s,a) \notin K_t$ is bounded by $2m_1(1+\frac{1}{(1-\gamma)\epsilon_1})$. On the other hand, once a state-action pair $(s,a)$ is experienced for $m_2$-th time at any timestep $t$, it will become a member of $K_t$ and will never leave $K_t$ anymore. So, $m_2$ is another upper-bound for the number of timesteps with $(s,a) \notin K_t$.

Generalizing the above fact for all state-action pairs, we conclude that the total number of escape events (timesteps $t$ with $(s_t,a_t) \notin K_t$) is bounded by $\min{(2 m_1 \kappa, |S||A| m_2)}$.

\end{document}